\documentclass{article}

\usepackage{amsmath}
\usepackage{amssymb}
\usepackage{amsthm}

\usepackage{rotating}

\newtheorem{theorem}{Theorem}

\usepackage{graphicx}
\usepackage{booktabs}
\usepackage{tikz}
\usepackage{wrapfig}

\newcommand\DCS{\mathrm{D}_{\textsc{cs}}}
\newcommand\Entropy{\mathrm{H}_2}
\newcommand\CrossEntropy{\mathrm{H}_2^\times}

\newcommand\manifold{\mathcal{V}}
\newcommand\tr{\mathrm{tr}}

\newcommand\R{\mathbb{R}}
\newcommand\nor{\mathcal{N}}
\newcommand\cov{\mathrm{cov}}
\newcommand\V{\mathrm{V}}
\newcommand\il[1]{\langle #1 \rangle}

\newcommand\minusOne{{_{-}}}
\newcommand\plusOne{{_{+}}}

\newcommand\for{\text{ for }}

\newcommand\X{\mathrm{X}}
\newcommand\Y{\mathrm{Y}}
\newcommand\x{\mathrm{x}}
\newcommand\y{\mathrm{y}}

\usepackage{stmaryrd}
\newcommand\de[1]{\llbracket  #1 \rrbracket}
\newcommand\fe[1]{\llbracket  #1 \rrbracket_{\nor}}

\newtheorem{optimizationproblem}{Optimization problem}
\newtheorem{observation}{Observation}

\usepackage{url}
\urldef{\mailsa}\path| {wojciech.czarnecki, jacek.tabor}@uj.edu.pl|     
\urldef{\mailsb}\path| rafjoz@gmail.com|

\begin{document}

%\mainmatter  % start of an individual contribution

\title{Maximum Entropy Linear Manifold\\for Learning Discriminative\\Low-dimensional  Representation}

%\titlerunning{MELM for Learning Discriminative Low-dimensional  Representation}

\author{Wojciech Marian Czarnecki$^1$, Rafal Jozefowicz$^2$  and Jacek Tabor$^1$}
\date{
$^1$Faculty of Mathematics and Computer Science,\\
Jagiellonian University, Krakow, Poland\\
\{wojciech.czarnecki, jacek.tabor\}@uj.edu.pl\\
$^2$Google, New York, USA\\
rafjoz@gmail.com}
%
%\authorrunning{WM Czarnecki, R Jozefowicz, J Tabor}

%\institute{$^1$Faculty of Mathematics and Computer Science,\\
%Jagiellonian University, Krakow, Poland\\
%\mailsa\\ \vspace{0.2cm}
%$^2$Google,\\
%New York, USA\\
%\mailsb
%}


%\toctitle{Lecture Notes in Computer Science}
%tocauthor{Authors' Instructions}
\maketitle

\begin{abstract}
Representation learning is currently a very hot topic in modern machine learning, mostly due to the great success of the deep learning methods.
In particular low-dimensional representation which discriminates classes can not only enhance the classification procedure, but also make it faster, while contrary to the high-dimensional embeddings can be efficiently used for visual based exploratory data analysis. 

In this paper we propose Maximum Entropy Linear Manifold (MELM), a multidimensional generalization of Multithreshold Entropy Linear Classifier model which is able to find a low-dimensional linear data projection maximizing discriminativeness of projected classes. As a result we obtain a linear embedding which can be used for classification, class aware dimensionality reduction and data visualization. MELM provides highly discriminative 2D projections of the data which can be used as a method for constructing robust classifiers.

We provide both empirical evaluation as well as some interesting theoretical properties of our objective function such us scale and affine transformation invariance, connections with PCA and bounding of the expected balanced accuracy error.
%\keywords{dense representation learning, data visualization, entropy, supervised dimensionality reduction}
\end{abstract}

\section{Introduction}
Correct representation of the data, consistent with the problem and used classification method, is crucial for the efficiency of the machine learning models. 
In practice it is a very hard task to find suitable embedding of many real-life objects in $\mathbb{R}^d$ space used by most of the algorithms. 
In particular for natural language processing~\cite{levy2014neural}, cheminformatics~\cite{smusz2015fingerprint} or even image recognition tasks it is still an open problem. 
As a result there is a growing interest in methods of representation learning~\cite{goodfellow2014challenges}, suited for 
finding better embedding of our data, which may be further used for classification, clustering or other analysis purposes. 
Recent years brought many success stories, such as dictionary
learning~\cite{mairal2009online} or deep learning~\cite{hinton2006fast}. 
Many of them look for a sparse~\cite{geng2014local}, highly dimensional embedding which simplify linear separation at a cost of making visual analysis nontrivial.
A dual approach is to look for low-dimensional linear embedding, which has advantage of easy visualiation, interpretation and manipulation at a cost of much weaker (in terms of models complexity) space of transformations.
% so user can analyze the data and possibly make further decisions regarding adding or removing some features. 

In this work we focus on the scenario where we are given labeled dataset in $\mathbb{R}^d$ and we are looking for such low-dimensional linear embedding which allows to easily distinguish each of the classes. In other words we are looking for a highly discriminative, low-dimensional representation of the given data.

\begin{wrapfigure}{r}{0.5\textwidth}
  \vspace{-20pt}
  \begin{center}
    \includegraphics[width=0.48\textwidth]{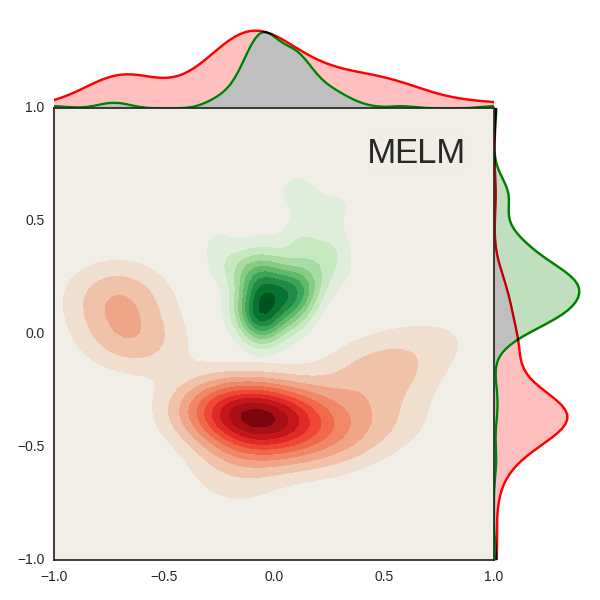}
  \end{center}
  \caption{Visualizatoin of \emph{sonar} dataset using Maximum Entropy Linear Manifold with $k=2$.}
  \label{fig:sonar}
  \vspace{-10pt}
\end{wrapfigure}

Our basic idea follows from the observation \cite{silverman1986density} that the density estimation is credible only in the low dimensional spaces. Consequently, we first project the data onto an arbitrary $k$-dimensional
affine submanifold $\manifold$ (where $k$ is fixed), and search for the $\manifold$ for which the estimated densities of the projected classes are orthogonal to each other, where  
the Cauchy-Schwarz Divergence is applied as a measure of discriminativeness of the projection, see Fig.~\ref{fig:sonar} for an example of such projection preserving classes' separation. The work presented in this paper is a natural extension of our earlier results~\cite{Czarnecki2015}, where we considered the one-dimensional case. However, we would like to underline that the used approach needed a nontrivial modification. 
In the one-dimensional case we could identify subspaces with elements of the unit sphere in a natural way. For higher dimensional subspaces such an identification is no longer possible.

To the authors best knowledge the presented idea is novel, and has not been earlier considered as a method of classification and data visualization.
As one of its benefits is the fact that it does not depend on 
affine rescaling of the data, which is a rare feature of the common
classification tools. What is also interesting, we show that as its simple limiting one-class case we obtain the classical PCA projection. Moreover, from the theoretical standpoint the Cauchy-Schwarz divergence factor can be decomposed into the fitting term, bounding the expected balanced misclassification error, and regularizing
term, simplifying the resulting model. We compute its value and derivative so one can use first-order optimization to find a solution
even though the true optimization should be performed on a Steifel manifold. Empirical tests show that such a method not only in some cases improves the classification score over learning from raw data but, more importantly, consistently finds highly discriminative representation which
can be easily visualized. In particular, we show that resulting projections' discriminativeness is much higher than many popular
linear methods, even recently proposed GEM model~\cite{GEM}.
For the sake of completness we also include the full source code of proposed method in the supplementary material.

\section{General idea}

In order to visualize dataset in $\mathbb{R}^d$ we need to project it onto $\mathbb{R}^k$ for very small $k$ (typically $2$ or $3$). One can use either linear transformation or some complex embedding, however choosing the second option in general leads to hard interpretability of the results. Linear projections have a tempting characteristics of being both easy to understand (from both theoretical perspective and practical implications of the obtained results) as well as they are highly robust in further application of this transformation. 

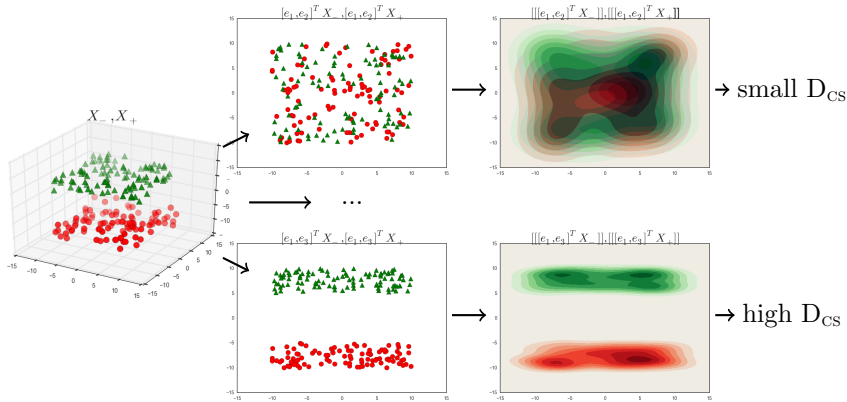
\begin{figure}[h]
\begin{tikzpicture}
\node[inner sep=0pt] (start) at (0,0)
    {\includegraphics[width=.3\textwidth]{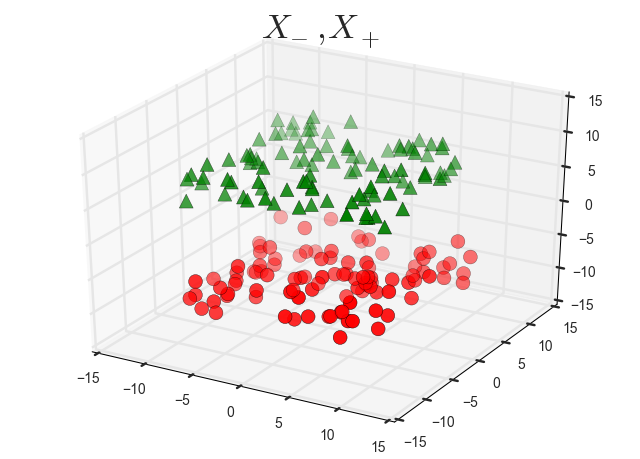}};
\node[inner sep=0pt] (proj1) at (3,1.5)
    {\includegraphics[width=.25\textwidth]{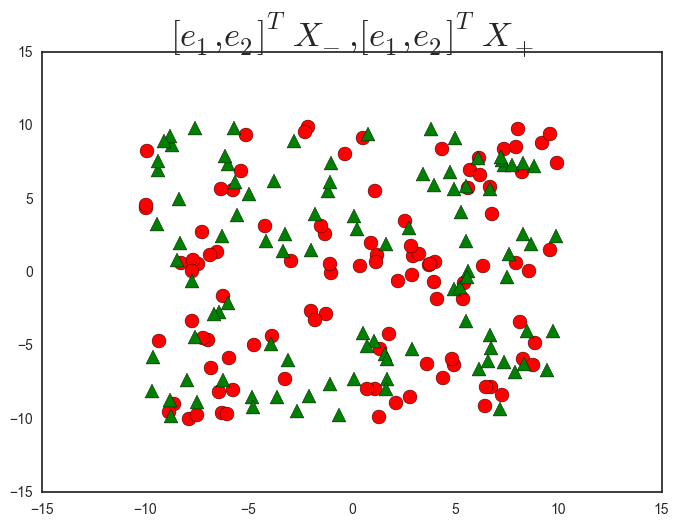}}    
    edge [->, thick] (start);

\node[inner sep=0pt] (proj2) at (3,0)
    {\;\;\;\;...}
    edge [<-, thick] (start);

\node[inner sep=0pt] (proj2) at (3,-1.5)
    {\includegraphics[width=.25\textwidth]{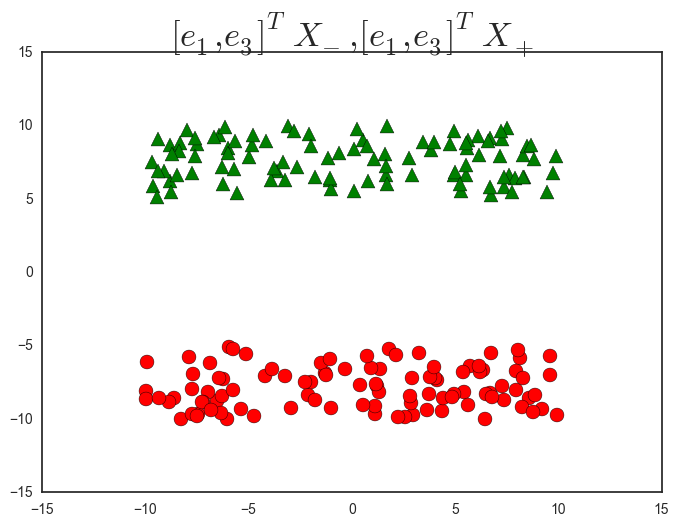}}
    edge [->, thick] (start);
\node[inner sep=0pt] (kde1) at (6.5,1.5)
    {\includegraphics[width=.25\textwidth]{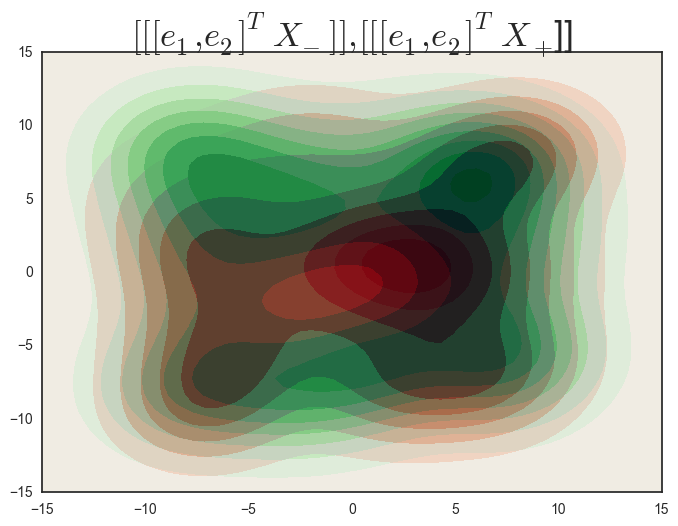}}
    edge [<-, thick] (proj1);
\node[inner sep=0pt] (kde2) at (6.5,-1.5)
    {\includegraphics[width=.25\textwidth]{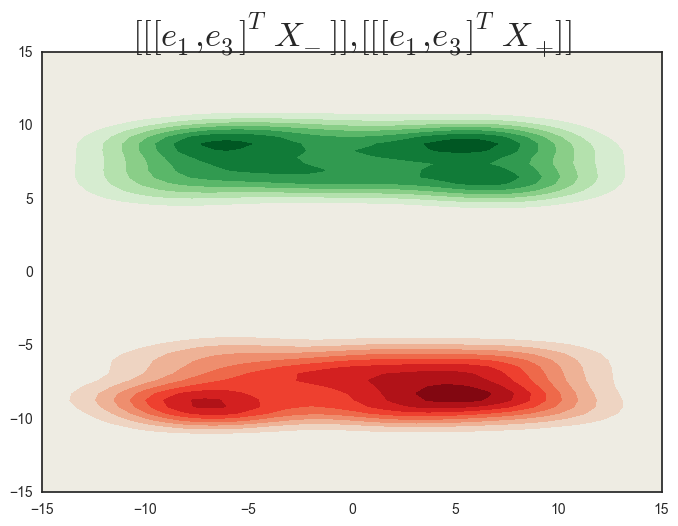}}
    edge [<-, thick] (proj2);
    
\node[inner sep=0pt] (dcs1) at (9,1.5)
    {\;small $\DCS$}
    edge [<-, thick] (kde1);
    
\node[inner sep=0pt] (dcs2) at (9,-1.5)
    {\;high $\DCS$}
    edge [<-, thick] (kde2);

%\draw[<->,thick] (russell.south east) -- (whitehead.north west)
%    node[midway,fill=white] {Principia Mathematica};
\end{tikzpicture}
\caption{Visualization of the MELM idea. For given dataset $\X\minusOne, \X\plusOne$ we search through various linear projections $\V$ and analyze how divergent are their density estimations in order to select the most discriminative.}
\label{fig:elcvis}
\end{figure}

In this work we focus on such class of projections so in practise we are looking for some matrix $\V \in \mathbb{R}^{d \times k}$, such that for a given dataset $\X \in \mathbb{R}^{d \times N}$ projection $\V^T\X$ preserves as much of the important information about $\X$ as possible (sometimes additionally under additional constraints). The choice of the definition of \emph{information measure} $\mathrm{IM}$ together with the set of constraints $\varphi_i$ defines a particular reduction 
method.
% \begin{equation*}
%  \begin{aligned}
% \underset{\V \in \mathbb{R}^{d \times k}}{\text{maximize}} & &\text{I}(\V^TX; X, Y)\\
% \text{s.t.}& &\varphi(\V) 
%  \end{aligned}
% \end{equation*}
\begin{equation*}
\begin{aligned}
& \underset{\V \in \mathbb{R}^{d \times k}}{\text{maximize}}
& & \mathrm{IM}(\V^T\X; \X, \Y) \\
& \text{subject to}
& & \varphi_i(\V), \; i = 1, \ldots, m.
\end{aligned}
\end{equation*}

There are many transformations which can achieve such results. For example, the well known Principal Component Analysis defines important information as data scattering so it looks for $\V$ which preserves as much of the $\X$ variance as possible and requires $\V$ to be orthogonal. In information bottleneck method one defines this measure as amount of mutual information between $\X$ and some additional $\Y$ (such as set of labels) which has to be preserved. Similar approaches are adapted in recently proposed Generalized Eigenvectors for Discriminative Features (GEM) where one tries to preserve the signal to noise ratio between samples from different classes. In case of Maximum Entropy Linear Manifold (MELM), introduced in this paper, important information is defined as the discriminativness of the samples from different classes with orthonormal $\V$. In other words we work with labeled samples (in general, binary labeled) and wish to preserve the ability to distinguish one class ($\X\minusOne$) from another 
($\X\plusOne$). 
In more formal terms, our optimization problem is to 
% $$
% \underset{\V \in \mathbb{R}^{d \times k}}{\text{maximize}}\;\; \DCS(\de{\V^TX\minusOne}, \de{\V^TX\plusOne})\text{, subject to } \V^T\V = I,
% $$
\begin{equation*}
\begin{aligned}
& \underset{\V \in \mathbb{R}^{d \times k}}{\text{maximize}}
& & \DCS(\de{\V^T\X\minusOne}, \de{\V^T\X\plusOne}) \\
& \text{subject to}
& &\V^T\V = I,
\end{aligned}
\end{equation*}
where $\DCS(\cdot,\cdot)$ denotes the Cauchy-Schwarz Divergence, the measure of how divergent are given probability distributions; $\de{ \cdot }$ denotes some density estimator which, given samples, returns a probability distribution. The general idea is also visualized on Fig.~\ref{fig:elcvis}.

\section{Theory}

% \subsection{Unsupervised learning}

We first discuss the one class case which has mainly introductory character as it shows the simplified version of our main idea. 

Suppose that we have unlabeled data $\X \subset \R^d$ and that we want 
to reduce the dimension of the data (for example to visualize it, reduce outliers, etc.) to $k < d$. One of the possible approaches is to use information theory
and search for such $k$-dimensional subspace $\manifold \subset \R^d$ for which the orthogonal projection of $\X$ onto $\manifold$ preserves as much information about $X$ as possible. 

One can clearly choose various measures of information. In our case, due to computational simplicity, we have decided to use Renyi's quadratic entropy, which for the density $f$ on $\R^k$ is given by
$$
\Entropy(f)=-\log \int_{\R^k} f^2(\x)d\x.
$$
One can equivalently use information potential~\cite{principe2000information}, which is given
as the $L^2$ norm of the density $\mathrm{ip}(f)=\int_{\R^k}f^2(\x)d\x$.
We need an easy observation that one can compute the Renyi's quadratic entropy for the normal density $\nor(m,\Sigma)$ in $\R^k$~\cite{cover2006elements}:
\begin{equation} \label{eq:enre}
\Entropy(\nor(m,\Sigma))=\tfrac{k}{2}\log(4\pi)+\tfrac{1}{2}\log(\det \Sigma).
\end{equation}
However, in order to compute the Renyi's quadratic entropy of the discrete data we first need to apply some density estimation technique.
By joining all the above mentioned steps together we are able to pose the basic optimization problem we are interested in.

\begin{optimizationproblem} 
Suppose that we are given data $\X$, and $k$ which denotes the dimension reduction. Find the orthonormal base $\V$ of the $k$-dimensional subspace\footnote{We identify those vectors with a linear space spanned over them.} $\manifold$ for which the value of
$$
\Entropy(\de{\V^T\X})
$$ 
is maximal, where $\de{\cdot}$ denotes a given fixed method of density estimation.
\end{optimizationproblem}

If we have data $\X$ with mean $m$ and covariance $\Sigma$ in $\R^d$ and $k$ orthonormal vectors $\V=[\V_1,\ldots,\V_k]$
then the we can ask what will be the mean and covariance of the orthogonal projection of $\X$ onto the space spanned by $\V$.
It is easy to show that it is given by $\V^Tm$ and $\V^T\Sigma \V$. In other words, if we consider data in the base given by 
orthonormal extension of $\V$ to the whole $\R^d$, the covariance of the projected data corresponds to the left upper $k \times k$ block submatrix of the original covariance.

We are going to show that if we apply the simplest density estimation of the underlying density for projected data given by the 
maximal likelihood estimator over the family of normal densities\footnote{That is for $A \subset \manifold$ we put $\fe{A}=\nor(m_A,\cov_A):\manifold \to \R_+$.}
then our optimization problem is equivalent to taking first $k$ elements of the base given by PCA.

\begin{theorem}
Let $\X \subset \R^d$ be a given dataset with mean $m$ and covariance $\Sigma$ and let $\fe{\cdot}$ denote the density estimation % which returns the normal density with corresponding mean and covariance. Let $k<d$ be fixed. Then 
which returns the maximum likelihood estimator over Gaussian densities. Then
$$
\max\{\Entropy(\fe{\V^T\X}) : \V \in \R^{d \times k}, \V^T\V = I\}
$$
is realized for the first $k$ orthonormal vectors given by the PCA.
\end{theorem}

\begin{proof}
By the comments before and \eqref{eq:enre} we have
$$
\Entropy(\fe{\V^T\X})=\tfrac{k}{2}\log(4\pi)+\tfrac{1}{2}\log(\det (\V^T \Sigma \V)).
$$
In other words we search for these $\V$ for which the value of $\det(\V^T \Sigma \V)$ is maximized. Now by Cauchy interlacing theory \cite{bhatia1997matrix} eigenvalues of $\V^T\Sigma \V$ (ordered decreasingly) are bounded above by the eigenvalues of $\Sigma$. Consequently, the maximum is obtained in the case when $\V$ denotes the orthonormal eigenvectors of $\Sigma$ corresponding to the biggest eigenvalues of $\Sigma$. However, this is exactly the first $k$ elements of the orthonormal base constructed by the PCA.
\end{proof}

\noindent Using analogous reasoning we can also prove the dual result.
\begin{theorem}
\label{th:h2min}
For $\X, m, \Sigma, \fe{\cdot}$ and $k$ as in the previous theorem
$$
\min\{\Entropy(\fe{\V^T\X}) : \V \in \R^{d \times k}, \V^T\V = I\}
$$
is realized for the last $k$ orthonormal vectors defined by the PCA.
\end{theorem}

As a result we obtain some general intuition that maximization of the Renyi's quadratic entropy leads to the selection of highly spreaded data, while its minimization selects projection where image is very condensed.

% \subsection{Supervised learning}

Let us now proceed to the binary labeled data. Recall that $\DCS$ can be equivalently expressed in terms of Renyi's quadratic entropy ($\Entropy$) and Renyi's quadratic cross entropy ($\CrossEntropy$):
$$
\begin{aligned}
\DCS(\V) &
=\log \int \de{\V^T\X\plusOne}^2+
\log \int \de{\V^T\X\minusOne}^2-2 \log \int \de{\V^T\X\plusOne}\de{\V^T\X\minusOne} \\[0.5ex]
&= - \Entropy(\de{\V^T\X\minusOne}) - \Entropy(\de{\V^T\X\plusOne}) + 2\CrossEntropy(\de{\V^T\X\plusOne}, \de{\V^T\X\minusOne}).
\end{aligned}
$$
Let us recall that our optimization aim is to find
a sequence $\V$ consisting of $k$ orthonormal vectors for which $\DCS(\V)$ is maximized.

\begin{observation}
Assume that the density estimator $\de{\cdot}$
does not change under the affine change of the coordinate system\footnote{This happens in particular for the kernel density estimation we apply in the paper}. One can show, by an easy modification of the theorem by Czarnecki and Tabor~\cite[Theorem 4.1]{Czarnecki2015}, that the maximum of $\DCS(\cdot)$ is independent of the affine change of data. Namely, for an arbitrary affine invertible map $M$ we have:
$$
\begin{aligned}
&\max\{\DCS(\V;\X\plusOne,\X\minusOne): \V \text{ orthonormal}\}\\[0.5ex]
=&\max\{\DCS(\V;\X\plusOne,\X\minusOne): \V \text{ linearly independent}\}\\[0.5ex]
=&\max\{\DCS(\V;M\X\plusOne,M\X\minusOne): \V \text{ orthonormal}\}.
\end{aligned}
$$
\end{observation}

The above feature, although typical in the density estimation, is rather uncommon in modern classification tools. 

Similarly to the one-dimensional case, when $\V \in \mathbb{R}^d$, we can decompose the objective function into fitting and regularizing terms:
$$
\DCS(\V) = \underbrace{2\CrossEntropy(\de{\V^T\X\plusOne}, \de{\V^T\X\minusOne})}_\text{fitting term} - \underbrace{(\Entropy(\de{\V^T\X\minusOne}) + \Entropy(\de{\V^T\X\plusOne}))}_\text{regularizing term}.
$$
Regularizing term has a slightly different meaning than in most of the machine learning models. Here it controls number of disjoint regions which appear after performing density based classification in the projected space. For one dimensional case it is a number of thresholds in the multithreshold linear classifier, for $k=2$ it is the number of disjoint curves defining decision boundary, and so on. Renyi's quadratic entropy is minimized when each class is as condensed as possible (as we show in Theorem~\ref{th:h2min}), intuitively resulting in a small number of disjoint regions.

It is worth noting that, despite similarities, it is not the common classification objective which can be written as an additive loss function and a regularization term
$$
L(\V) = \sum_{i=1}^N \ell(\V^T\x_i, \y_i, \x_i) + \mathrm{\Omega}(\V),
$$
as the error depends on the relations between each pair of points instead of each point independently. One can easily prove that there are no $\ell, \Omega$ for which $\DCS(v) = L(\V; \ell, \mathrm{\Omega})$. Such choice of the objective function might lead to the lack of connections with optimization of any reasonable accuracy related metric, as those are based on the point-wise loss functions. However it appears that $\DCS$ bounds the expected balanced accuracy\footnote{Accuracy with class priors being ignored $\text{BAC} = \tfrac{1}{2}\left ( \tfrac{\text{TP}}{\text{TP}+\text{FN}}  + \tfrac{\text{TN}}{\text{TN} + \text{FP}} \right )$.} similarly to how hinge loss bounds 0/1 loss. This can be formalized in the following way.

\begin{theorem}
 Negative log-likelihood of balanced misclassification in $k$-dimensional linear projection of any non-separable densities $f_\pm$ onto $\V$ is bounded by half of the Renyi's quadratic cross entropy of these projections.
\end{theorem}
\begin{proof}
 Likelihood of balanced misclassification over a $k$-dimensional hypercube after projection through $\V$ equals
 $$
 \int_{[0,1]^k} \min \{ (\V^Tf_+)(\x) , (\V^Tf_-)(\x)  \}d\x.
 $$
 
 Using analogous reasoning to the one presented by Czarnecki~\cite{Czarnecki2015cons}, using Cauchy and other basic inequalities, one can show that
 $$
 - \log \int_{[0,1]^k} \min \{ (\V^Tf_+)(\x) , (\V^Tf_-)(\x)  \}d\x \geq \tfrac{1}{2} \CrossEntropy (\V^Tf_+, \V^Tf_-).
 $$
 \qed
\end{proof}
As a result we might expect that maximizing of the $\DCS$ leads to the selection of the projection which on one hand maximizes the balanced accuracy over the training set (minimizes empirical error) and on the other fights with overfitting by minimizing the number of disjoint classification regions (minimizes model complexity).

%%%%%%%%%%%%%%%%%%%%%%%%%%%%%%%%%%%%%%%%%%%%
% \section{Computations}
\section{Closed form solution for objective and its gradient}

Let us now investigate more practical aspects of proposed approach. We show the exact formulas of both $\DCS$ and its gradient as functions of finite, labeled samples (binary datasets)
so one can easily plug it in to any first-order optimization software.

%%%%%%%%%%%%%%%%%%%%%%%%%%%%%%%%%%%%%%%%%%%%
%  \section{Divergence}

Let $\X\plusOne,\X\minusOne$ be fixed subsets of $\R^d$.
Let $\manifold$ denote the $k$-dimensional subspace generated
by $\V=[\V_1,\ldots,\V_k] \in \R^{d \times k}$ (we consider only the case when the 
sequence $\V$ is linearly independent). We project sets $\X_\pm$ orthogonally
on $\V$, and compute the Cauchy-Schwarz Divergence 
of the kernel density estimations (using Silverman's rule) of 
the resulting projections:
$$
\mathrm{G}^{-1}(\V)\de{\V^T\X\plusOne} \text{ and }\mathrm{G}^{-1}(\V)\de{\V^T\X\minusOne},
$$
where $\mathrm{G}(\V)=\V^T\V$ denotes the grassmanian. We search for such $\V$ for which the
Cauchy-Schwarz Divergence is maximal. Recall that the scalar product in the space of matrices is given by $\il{\V_1,\V_2}=\tr(\V_1^T\V_2)$.

There are basically two possible approaches one can apply: either search for the solution in the set of orthonormal $\V$ which generate $\manifold$, 
or allow all $\V$ with a penalty function. The first method is possible\footnote{And has advantage of having smaller number of parameters.}, 
but does not allow use of most of the existing numerical libraries as the space we work in is highly nonlinear. This is the reason why we use the
second approach which we describe below.

Since, as we have observed in the previous section, the
result does not depend on the affine transformation of data, we
can restrict to the analogous formula for the sets
$$
\V^T\X\plusOne \text{ and }\V^T\X\minusOne,
$$
where $\V$ consists of linearly independent vectors.
Consequently, we need to compute the gradient of the function
$$
\begin{aligned}
\DCS(\V) &=\DCS(\de{\V^T\X\plusOne},\de{\V^T\X\minusOne}) \\[0.5ex]
&=\log \int \de{\V^T\X\plusOne}^2+
\log \int \de{\V^T\X\minusOne}^2-2 \log \int \de{\V^T\X\plusOne}\de{\V^T\X\minusOne},
\end{aligned}
$$
where we consider the space consisting only of linearly independent vectors. 
Since as the base of the space $\V$ we can always take orthonormal vectors, the maximum
is realized for orthonormal sequence, and therefore we can add a penalty term for being 
non-orthonormal sequence, which helps avoiding numerical instabilities:
$$
\DCS(\V)-\|\V^T\V-I\|^2,
$$
where as we recall the sequence $\V$ is orthonormal iff $\V^T\V=I$. We denote above augmented $\DCS$ by the \emph{maximum entropy linear manifold} objective function
\begin{equation} \label{ref:eq0}
\mathrm{MELM}(\V) = \DCS(\V)-\|\V^T\V-I\|^2.
\end{equation}

Besides $\mathrm{MELM}(\cdot)$ value we need the formula for its gradient $\nabla \mathrm{MELM}(\cdot)$. 
For the second term we obviously have
$$
\nabla \|\V^T\V-I\|^2=4\V\V^T\V-4\V.
$$

We consider the first term. Let us first provide the formula for the computation
of the product of kernel density estimations of two sets.

Assume that we are given set $A \subset \manifold$ (in our case $A$ will be the projection of $\X_\pm$ onto $V$), where $\manifold$ is $k$-dimensional. 
Then the formula for the kernel density estimation with Gaussian kernel, is given by~\cite{silverman1986density}:
$$
\de{A}=\frac{1}{|A|} \sum_{a \in A}\nor(a,\Sigma_A),
$$
where $\Sigma_A={(h^\gamma_A)}^2\cov_A$ and (for $\gamma$ being a scaling hyperparameter~\cite{Czarnecki2015}) %[str 86-87]
$h^\gamma_A=\gamma(\tfrac{4}{k+2})^{1/(k+4)}|A|^{-1/(k+4)}$.

Now we need the formula for $\int \de{A}\de{B}$, which is calculated~\cite{Czarnecki2015}
with the use of 
$$
\int \nor(a,\Sigma_A) \nor(b,\Sigma_B)=\nor(a-b,\Sigma_A+\Sigma_B)(0).
$$
Then we get
$$
\begin{aligned}
\int \de{A}\de{B} &=\frac{1}{|A||B|} \sum_{w \in A-B}\nor(w,\Sigma_{A}+\Sigma_B)(0)\\
&=\frac{1}{(2\pi)^{k/2}\det^{1/2}(\Sigma_{AB})|A||B|} \sum_{w \in A-B} \exp(-\tfrac{1}{2}\|w\|^2_{\Sigma_{AB}}),
\end{aligned}
$$
where $A-B=\{a-b:a \in A,b \in B\}$ and $\Sigma_{AB}$ is defined by
$$
\begin{aligned}
\Sigma_{AB} &={(h^\gamma_{A})}^2 \cov_{A}+{(h^\gamma_{B})}^2 \cov_{B}\\
&= \gamma^2(\tfrac{4}{k+2})^{2/(k+4)}(|A|^{-2/(k+4)}\cov_{A}+|B|^{-2/(k+4)}\cov_{B}).
\end{aligned}
$$
For a sequence $\V=[\V_1,\ldots,\V_k] \in \R^{d \times k}$ of linearly independent vectors we put
$$
\Sigma_{AB}(\V)=\V^T\Sigma_{AB}\V \text{ and } S_{AB}(\V)=\Sigma_{AB}(\V)^{-1}.
$$
Observe that $\Sigma_{AB}(\V)$ and $S_{AB}(\V)$ are
square symmetric matrices which represent the properties of the projection of the data onto the space spanned over $\V$.
We put 
$$
\phi_{AB}(\V)=\frac{1}{(2\pi)^{k/2}\det^{1/2}(\Sigma_{AB}(\V))|A||B|},
$$
thus
$$
\nabla \phi_{AB}(\V)=-\phi_{AB}(\V) \cdot \Sigma_{AB} \cdot  \V  \cdot S_{AB}(\V).
$$
Consequently to compute the final formula, we need the gradient of the function $\V \to \det(\Sigma_{AB}(\V))$, 
which as one can easily verify, is given by the formula
\begin{equation} \label{eq2}
\nabla \det(\Sigma_{AB}(\V))=2\det(\V^T\Sigma_{AB} \V)
\cdot \Sigma_{AB} \V (\V^T\Sigma_{AB} \V)^{-1}.
\end{equation}
One can also easily check that for 
$$
\psi^w_{AB}(\V)=\exp(-\tfrac{1}{2}\|\V^Tw\|^2_{\Sigma_{AB}(\V)}),
$$
where $w$ arbitrarily fixed, we get
$$
\nabla \psi^w_{AB}(\V)
=-\psi^w_{AB}(\V) \cdot (ww^T \V S_{AB}(\V)-\Sigma_{AB}  \V 
S_{AB}(\V) \V^T w w^T \V S_{AB}(\V)).
$$ 

To present the final form for the gradient of $\DCS(\V)$ we need
the gradient of the cross information potential 
$$
\begin{aligned}
\text{ip}^\times_{AB}(\V)&=\phi_{AB}(\V) \sum \limits_{w \in A-B}\psi_{AB}^w(\V),\\
\nabla \text{ip}^\times_{AB}(\V)&=\phi_{AB}(\V) \sum \limits_{w \in A-B}\nabla \psi^w_{AB}(\V)
+\left( \sum \limits_{w \in A-B} \psi_{AB}^w(\V) \right) \cdot \nabla \phi_{AB}(\V).
\end{aligned}
$$
Since 
$$
\DCS(\V)=\log(\text{ip}^\times_{\X\plusOne\X\plusOne}(\V))+\log(\text{ip}^\times_{\X\minusOne\X\minusOne}(\V))-2\log(\text{ip}^\times_{\X\plusOne\X\minusOne}(\V)),
$$
we finally get
$$
\begin{aligned}
\nabla \DCS(\V)=&\tfrac{1}{\text{ip}^\times_{\X\plusOne\X\plusOne}(\V)} \nabla \text{ip}^\times_{\X\plusOne\X\plusOne}(\V) + \tfrac{1}{\text{ip}^\times_{\X\minusOne\X\minusOne}(\V)} \nabla \text{ip}^\times_{\X\minusOne\X\minusOne}(\V)\\
&- 2 \tfrac{1}{\text{ip}^\times_{\X\plusOne\X\minusOne}(\V)} \nabla\text{ip}^\times_{\X\plusOne\X\minusOne}(\V).
\end{aligned}
$$
Given 
$$
\begin{aligned}
\mathrm{MELM}(\V) &= \DCS(\V) - \|\V^T\V - I \|^2,\\
\nabla \mathrm{MELM}(\V) &= \nabla \DCS(\V) - (4\V\V^T\V - 4\V),
\end{aligned}
$$
 one can run any first-order optimization 
method to find vectors $\V$ spanning $k$-dimensional subspace $\manifold$ representing low-dimensional, discriminative manifold of the input space.

\section{Experiments}

We use ten binary classification datasets from UCI repository~\cite{UCI} and libSVM repository~\cite{chang2011libsvm}, which are briefly summarized in Table~\ref{tab:summary}. These are moderate size problems.
%with at most $60$ dimensions. 

Code was written in Python with the use of scikit-learn~\cite{pedregosa2011scikit}, numpy~\cite{van2011numpy} and scipy. Besides MELM we use 8 other linear dimensionality reduction techniques, namely: Principal Component Analysis (PCA), class PCA (cPCA\footnote{cPCA uses sum of each classes covariances, weighted by classes sizes, instead of whole data covariance.}), two ellipsoid PCA (2ePCA\footnote{2ePCA is cPCA without weights, so it is a balanced counterpart.}), per class PCA (pPCA\footnote{pPCA uses as $\V_i$ the first principal component of $i$th class.}), Independent Component Analysis (ICA), Factor Analysis (FA), Nonnegative Matrix Factorization (NMF\footnote{In order to use NMF we first transform dataset so it does not contain negative values.}), Disriminative Learning using Generalized Eigenvectors (GEM~\cite{GEM}). PCA, ICA, NMF and FA are implemented in scikit-learn, cPCA, pPCA and 2ePCA were coded by authors and for GEM we use publically available code\footnote{forked at 
\url{http://gist.github.com/lejlot/3ab46c7a249d4f375536}}. Implementation of MELM as a model compatible with scikit-learn classifiers and transformers is available both in supplementary materials and online\footnote{\url{http://github.com/gmum/melm}}.

\begin{table}[h]
\begin{center}
\begin{tabular}{lrrrrrrrr}
\toprule
dataset & $N$ & $d$ & $|\X\minusOne|$ & $|\X\plusOne|$ & $\hat m$ & $d^{.95}$ & $d_-^{.95}$ & $d_+^{.95}$\\
\midrule
australian & 690 & 14 & 383 & 307 & 0.80 & 1 & 2 & 1 \\
breast-cancer & 683 & 10 & 444 & 239 & 1.00 & 1 & 1 & 1 \\
diabetes & 768 & 8 & 268 & 500 & 0.88 & 2 & 2 & 3 \\
fourclass & 862 & 2 & 555 & 307 & 1.00 & 2 & 2 & 2 \\
german.numer & 1000 & 24 & 700 & 300 & 0.75 & 3 & 3 & 3 \\
heart & 270 & 13 & 150 & 120 & 0.75 & 3 & 3 & 3 \\
ionosphere & 351 & 34 & 126 & 225 & 0.88 & 24 & 26 & 7 \\
liver-disorders & 345 & 6 & 145 & 200 & 1.00 & 3 & 3 & 3 \\
sonar & 208 & 60 & 111 & 97 & 1.00 & 28 & 24 & 24 \\
splice & 1000 & 60 & 483 & 517 & 1.00 & 55 & 52 & 54 \\
\bottomrule
\end{tabular}
\caption{Summary of used datasets. $N$ denote number of points, $d$ dimensionality, $|\X_l|$ number of samples with $l$ label, $\bar m$ mean density (number of nonzero elements) and $d^t_l$ denotes number of dimensions which we have to include during PCA to keep $t$ of label $l$ variance.}
\label{tab:summary}
\end{center}
\end{table}

% 
% \begin{table}
% \begin{center}
% \begin{tabular}{lrrrrrrrr}
% \toprule
%   & 2ePCA & cPCA & MELM & FA & GEM & I & PCA & \\
% \midrule
% australian & 0.798 & 0.790 & \textbf{0.870} & 0.847 & 0.791 & 0.860 & 0.839 & \\
% breast-cancer & 0.971 & 0.976 & 0.976 & 0.973 & 0.949 & 0.966 & \textbf{0.98} & \\
% diabetes & 0.690 & 0.695 & \textbf{0.761} & 0.737 & 0.637 & 0.728 & 0.726 & \\
% fourclass & \textbf{1.00} & \textbf{1.00} & \textbf{1.00} & 0.720 & \textbf{1.00} & \textbf{1.00} & \textbf{1.00} & \\
% german.numer & 0.642 & 0.668 & \textbf{0.705} & 0.635 & 0.659 & \textbf{0.728} & 0.700 & \\
% heart & 0.793 & 0.782 & \textbf{0.831} & 0.821 & 0.719 & \textbf{0.837} & 0.828 & \\
% ionosphere & 0.865 & 0.872 & \textbf{0.892} & 0.874 & 0.846 & \textbf{0.944} & 0.891 & \\
% liver-disorders & 0.586 & 0.594 & \textbf{0.710} & 0.570 & 0.682 & 0.705 & 0.569 & \\
% sonar & 0.818 & 0.804 & 0.790 & 0.814 & \textbf{0.892} & 0.862 & 0.824 & \\
% splice & 0.697 & 0.697 & \textbf{0.902} & 0.777 & 0.814 & 0.887 & 0.763 & \\
% \bottomrule
% \end{tabular}
% \end{center}
% \caption{ Comparison of low-dimensional (up to 4 dimensions) reduction followed by the classifier }
% \end{table}

In order to estimate how hard to optimize is the MELM objective function we plot in Fig.~\ref{fig:starts} histograms of $\DCS$ values obtained during 500 random starts for each of the dataset. First, one can easily notice that $\DCS$ have multiple local extrema (see for example \emph{heart} or \emph{liver-disorders} histograms). It also appears that in some of the considered datasets it is not easy to obtain maximum by the use of completely random starting point (see \emph{ionosphere} and \emph{australian} datasets), which suggests that one should probably use some more advanced initialization techniques.
\begin{figure}[h]
\begin{center}
\includegraphics[width=0.18\textwidth]{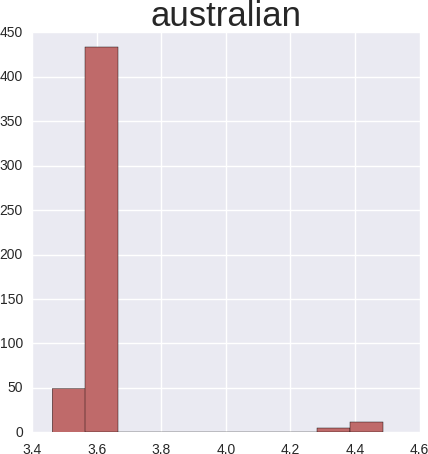}
\includegraphics[width=0.18\textwidth]{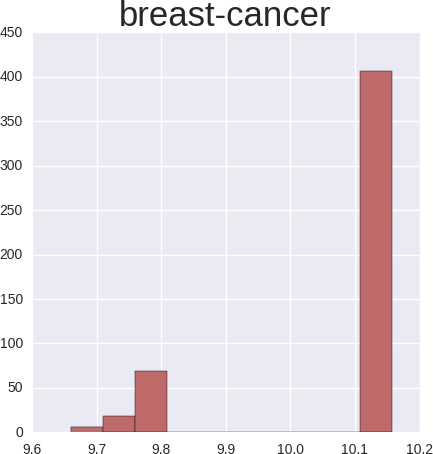}
\includegraphics[width=0.18\textwidth]{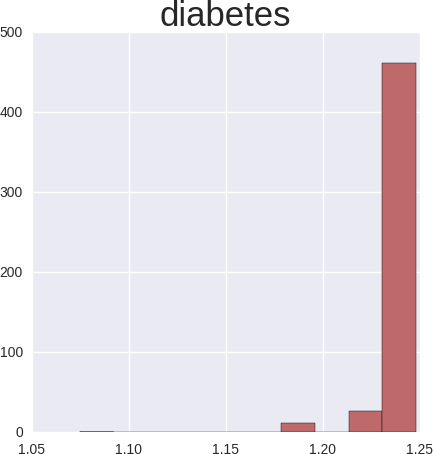}
\includegraphics[width=0.18\textwidth]{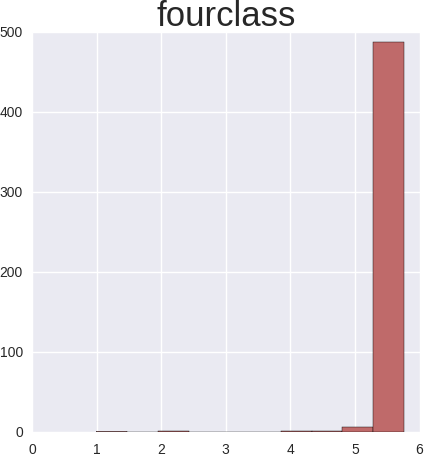}
\includegraphics[width=0.18\textwidth]{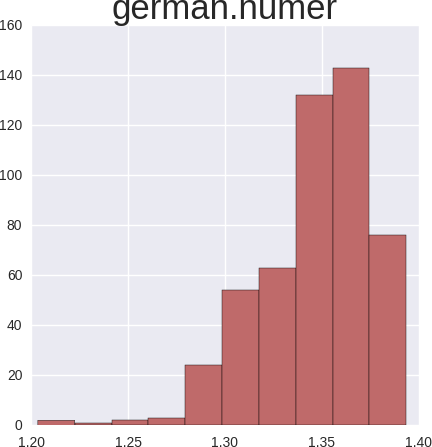}\\
\includegraphics[width=0.18\textwidth]{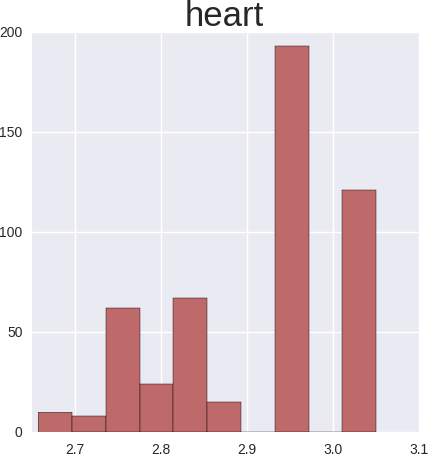}
\includegraphics[width=0.18\textwidth]{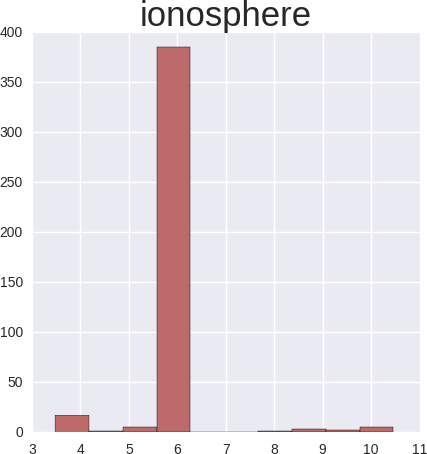}
\includegraphics[width=0.18\textwidth]{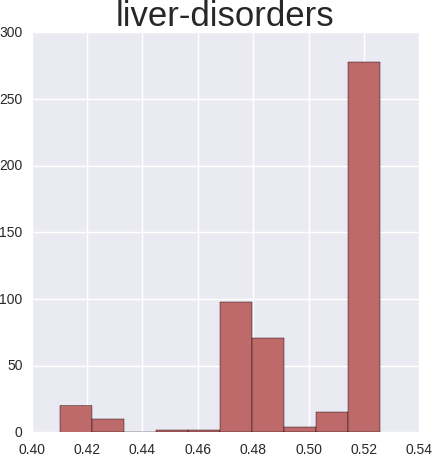}
\includegraphics[width=0.18\textwidth]{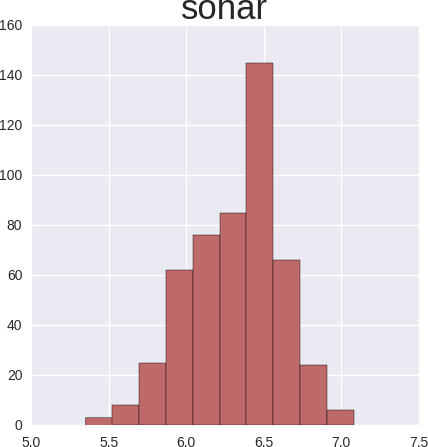}
\includegraphics[width=0.18\textwidth]{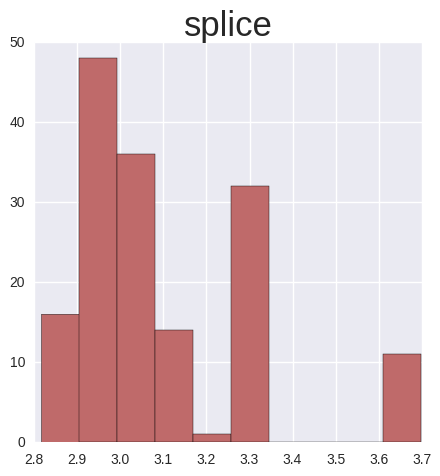}
\caption{Histograms of $\DCS$ values obtained for each dataset during $500$ random starts using L-BFGS.}
\label{fig:starts}
\end{center}
\end{figure}

To further investigate how hard it is to find a good solution when selecting maximum of $\DCS$ we estimate the expected value of $\DCS$ after $s$ random starts from matrices $V^{(1)}, \dots, V^{(s)}$
$$
\mathbb{E}[\max_{\V = \V^{(1)},\dots,\V^{(s)}} \DCS(\text{L-BFGS}(\text{MELM}| \V))].
$$
As one can see on Fig.~\ref{fig:iters} for 8 out of 10 considered datasets one can expect to find the maximum (with 5\% error) after just 16 random starts. Obviously this cannot be used as a general heuristics as it is heavily dependent on the dataset size, dimensionality as well as its discriminativness. However, this experiment shows that for moderate size problems (hundreds to thousands samples with dozens of dimensions) MELM can be relatively easily optimized even though it is a rather complex function with possibly many local maxima.
\begin{figure}[h]
\begin{center}
\includegraphics[width=0.5\textwidth]{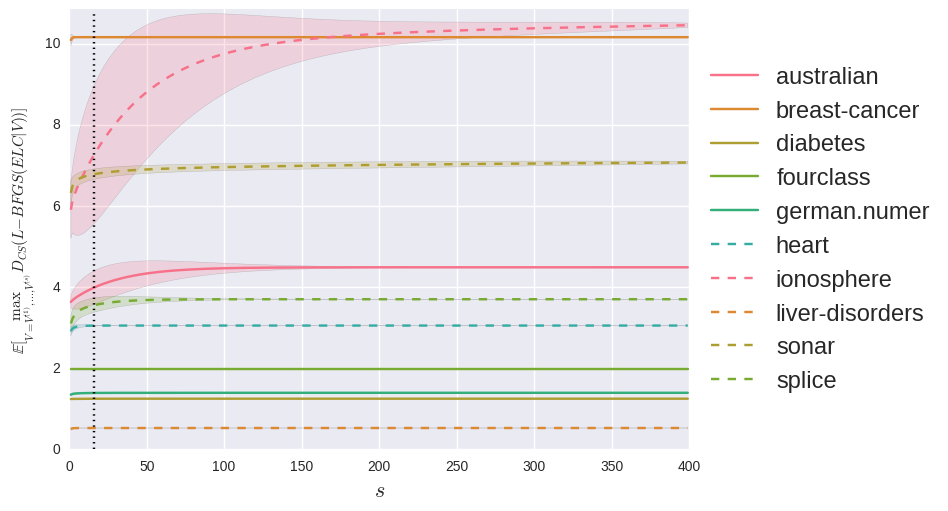}
\includegraphics[width=0.43\textwidth]{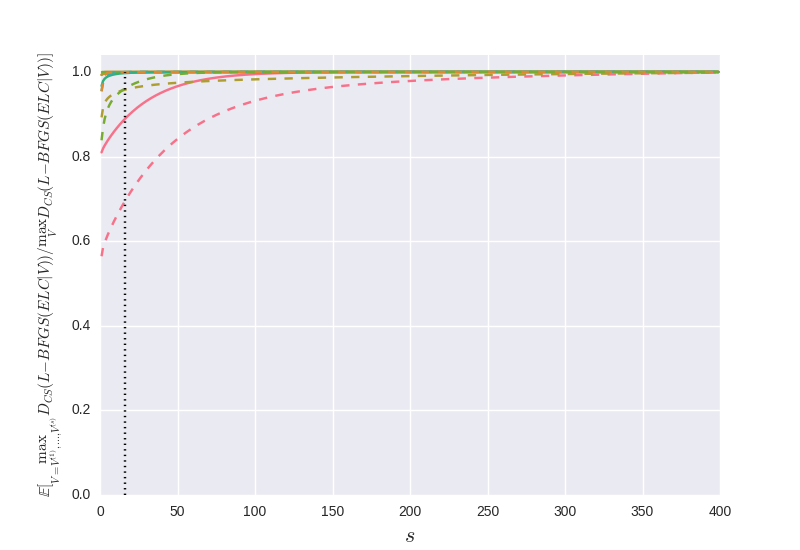}
\caption{Expected value of Cauchy-Schwarz Divergence after MELM optimization for $s$ random starts using L-BFGS algorithm (on the left) and its ratio to the maximum obtainable Cauchy-Schwarz Divergence (on the right). Dotted black line shows 16 starts threshold. }
\label{fig:iters}
\end{center}
\end{figure}

It is worth noting that truly complex optimization problem is only given by \emph{ionosphere} dataset. One can refer to Table~\ref{tab:summary} to see that this is a very specific problem where positive class is located in a very low dimensional linear manifold (approximately 7 dimensional) while the negative class is scattered over nearly 4 times more dimensions. 

We check how well MELM behaves when used in a classification pipeline. There are two main reasons for such approach, first if the discriminative manifold is low-dimensional, searching for it may boost the classification accuracy. Second, even if it decreases classification score as compared to non-linear methods applied directly in the input space, the resulting model will be much simpler and more robust. For comparison consider training a RBF SVM in $\R^{60}$ using $1000$ data points. It is a common situation when SVM selects large part of the dataset as the support vectors~\cite{suykens2002least},~\cite{wang2005support}, meaning that the classification of the new point requires roughly $500 \cdot 60 = 30000$ operations. In the same time if we first embed space in a plane and fit RBF SVM there we will build a model with much less support vectors (as the 2D decision boundary generally is not as complex as 60-dimensional one), lets say $100$ and consequently we will need $60 \cdot 2 + 2 \cdot 100 = 120 + 200 
= 
320$ operations, two orders of 
magnitude faster.
 Whole pipeline is composed of:
\begin{enumerate}
 \item Splitting dataset into training $\X\minusOne, \X\plusOne$ and testing $\hat \X\minusOne, \hat \X\plusOne$.
 \item Finding plane embeding matrix $\V \in \mathbb{R}^{d \times 2}$ using tested method.
 \item Training a classifier $cl$ on $\V^T \X\minusOne, \V^T \X\plusOne$.
 \item Testing $cl$ on $\V^T \hat \X\minusOne, \V^T \hat \X\plusOne$.
\end{enumerate}

Table~\ref{tab:2d} summarizes BAC scores obtained by each method on each of the considered datasets in 5-fold cross validation. For the classifier module we used SVM RBF, KNN and KDE-based density classification. Each of them was fitted using internal cross-validation to find the best parameters. GEM and MELM $\gamma$ hyperperameters were also fitted. Reported results come from the best classifier. 

\begin{sidewaystable}[!htbp]
\begin{center}
\begin{tabular}{lrrrrrrrrrrr}
\toprule
  & MELM & FA & ICA & GEM & NMF & 2ePCA & cPCA & PCA & pPCA & I & \\
\midrule
australian & \textbf{0.866} & 0.847 & 0.829 & 0.791 & 0.817 & 0.764 & 0.756 & 0.825 & 0.769 & 0.860 & \\
breast-cancer & \textbf{0.976} & 0.973 & \textbf{0.976} & 0.930 & \textbf{0.976} & 0.966 & 0.967 & \textbf{0.976} & 0.961 & 0.966 & \\
diabetes & \textbf{0.744} & 0.682 & 0.705 & 0.637 & 0.704 & 0.689 & 0.695 & 0.705 & 0.646 & 0.728 & \\
fourclass & \textbf{1.0} & 0.720 & \textbf{1.0} & \textbf{1.0} & 0.999 & \textbf{1.0} & \textbf{1.0} & \textbf{1.0} & \textbf{1.0} & \textbf{1.0} & \\
german.numer & \textbf{0.705} & 0.588 & 0.648 & 0.653 & 0.63 & 0.588 & 0.602 & 0.650 & 0.619 & \textbf{0.728} & \\
heart & \textbf{0.831} & 0.792 & 0.818 & 0.675 & 0.811 & 0.793 & 0.782 & 0.817 & 0.787 & \textbf{0.837} & \\
ionosphere & \textbf{0.892} & 0.794 & 0.757 & 0.763 & 0.799 & 0.783 & 0.780 & 0.757 & 0.826 & \textbf{0.944} & \\
liver-disorders & \textbf{0.710} & 0.546 & 0.545 & 0.681 & 0.553 & 0.531 & 0.548 & 0.531 & 0.557 & 0.705 & \\
sonar & 0.766 & 0.558 & 0.600 & \textbf{0.889} & 0.657 & 0.593 & 0.575 & 0.600 & 0.676 & 0.862 & \\
splice & \textbf{0.862} & 0.718 & 0.697 & 0.799 & 0.691 & 0.686 & 0.686 & 0.697 & 0.694 & \textbf{0.887} & \\
\bottomrule
\end{tabular}
\end{center}
\caption{ Comparison of 2-dimensional reduction followed by the classifier. I stands for Identity, meaning that we simply trained classifiers directly on the raw data, without any dimensionality reduction. Bold values indicate the best score obtained across all dimensionality reduction pipelines. If the classifier trained on raw data is better than any of the reduced models than its score is also bolded.}
\label{tab:2d}
\end{sidewaystable}

In four datasets, MELM based embeding led to the construction of better classifier than both other dimensionality reduction techniques as well as training models on raw data. This suggests that for these datasets the discriminative manifold is truly at most 2-dimensional. At the same time in nearly all (besides \emph{sonar}) datasets the pipeline consisting of MELM yielded significantly better classification results than any other embeding considered.

One of the main applications of MELM is to visualize the dataset through linear projection in such a way that classes do not overlap. One can see comparisons of \emph{heart} dataset projections using all considered approaches in Fig.~\ref{fig:2d}. As one can notice our method finds plane projection where classes are nearly perfectly discriminated. Interestingly, this separation is only obtainable in two dimensions, as neither marginal distributions nor any other one-dimensional projection can construct such separation.

\begin{figure}[!htbp]
\begin{center}
\includegraphics[width=0.32\textwidth]{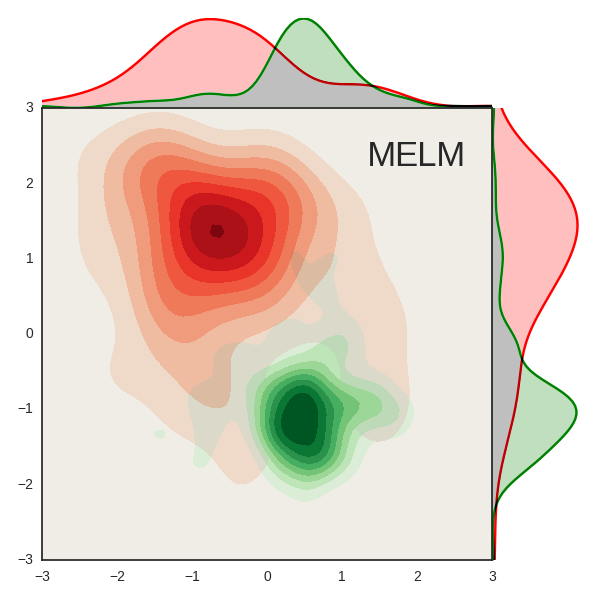}
\includegraphics[width=0.32\textwidth]{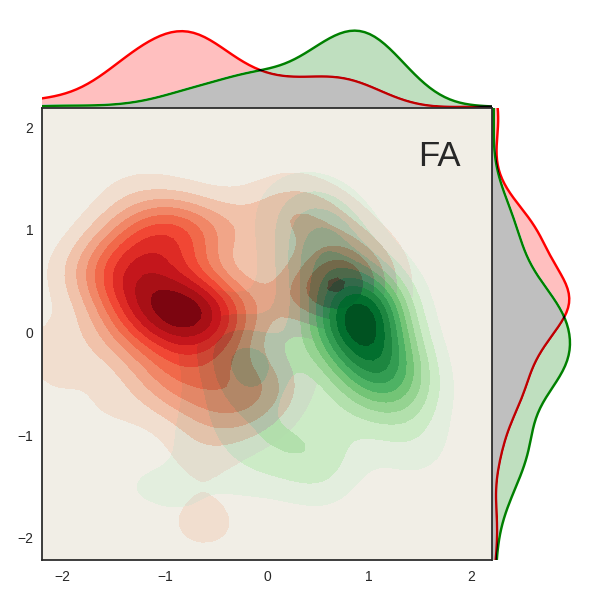}
\includegraphics[width=0.32\textwidth]{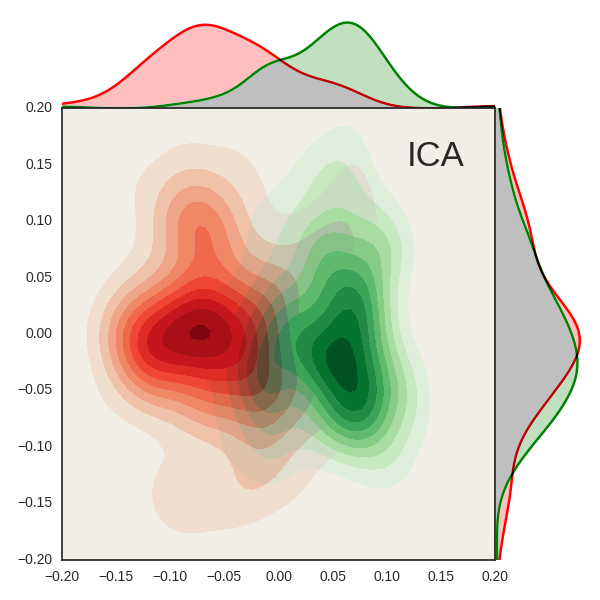}\\
\includegraphics[width=0.32\textwidth]{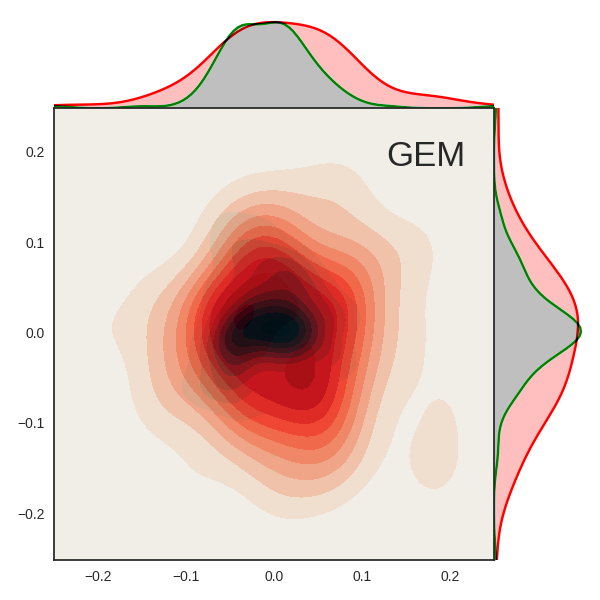}
\includegraphics[width=0.32\textwidth]{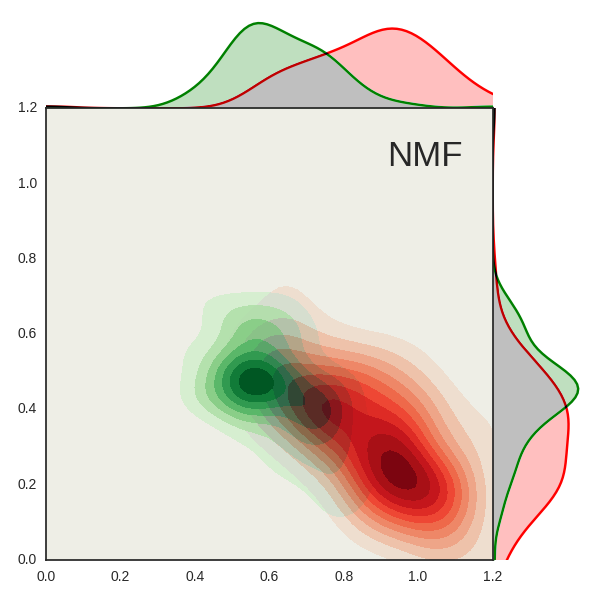}
\includegraphics[width=0.32\textwidth]{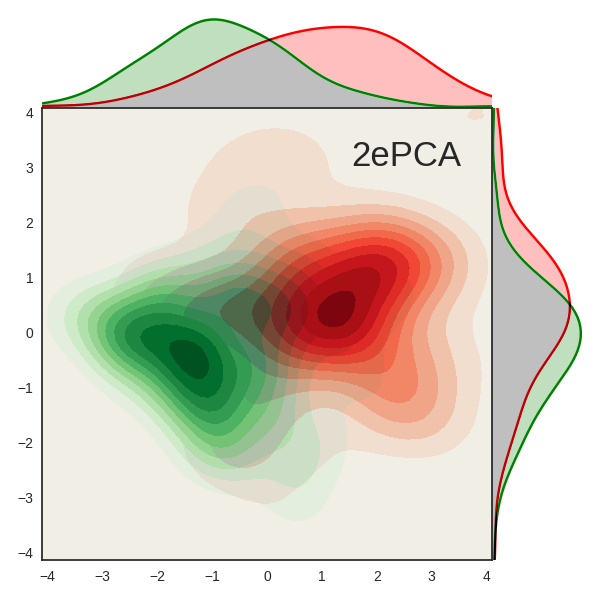}\\
\includegraphics[width=0.32\textwidth]{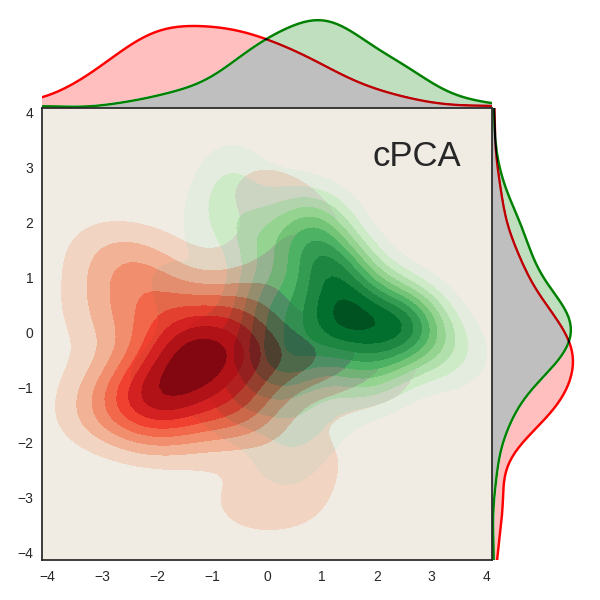}
\includegraphics[width=0.32\textwidth]{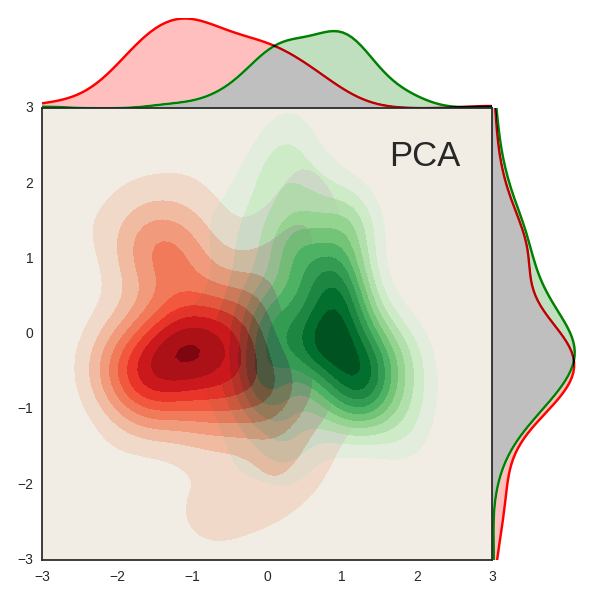}
\includegraphics[width=0.32\textwidth]{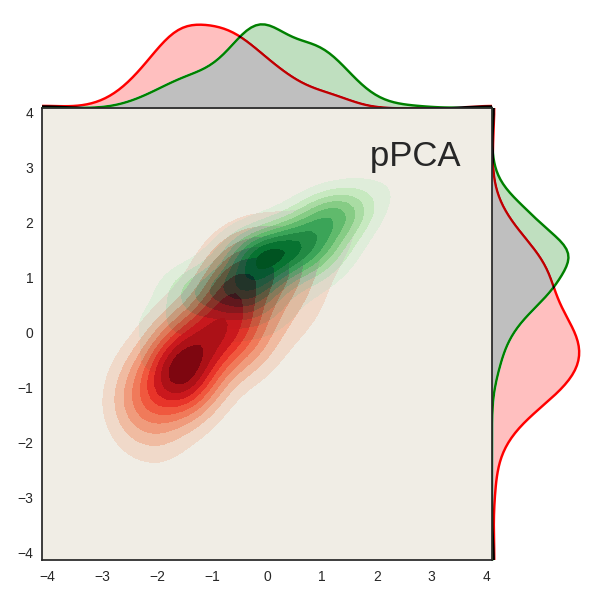}
\end{center}
\caption{Comparison of heart dataset 2D projections by analyzed methods. Visualization uses kernel density estimation.}
\label{fig:2d} 
\end{figure}

While visual inspection is crucial for such tasks, to truly compare competetive methods we need some metric to measure quality of the visualization. 
In order to do so, we propose to assign a \emph{visual separability} score as the mean BAC score over three considered classifiers (SVM RBF, KNN, KDE) 
trained and tested in 5-fold cross validation of the projected data. The only difference between this test and the previous one is that we use whole 
data to find a projection (so each projection technique uses all datapoints) and only further visualization testing is performed using train-test splits. 
This way we can capture "how easy to discriminate are points in this projection" rather than "how useful for data discrimination is using the projection". 
Experiments are repeated using various random subsets of samples and mean results are reported.

\begin{sidewaystable}[!htbp]
\begin{center}
\begin{tabular}{lrrrrrrrrrrr}
\toprule
  & MELM & FA & ICA & GEM & NMF & 2ePCA & cPCA & PCA & pPCA & \\
\midrule
australian & \textbf{0.888} & 0.856 & 0.845 & 0.792 & 0.838 & 0.782 & 0.781 & 0.845 & 0.792 & \\
breast-cancer & \textbf{0.985} & 0.979 & 0.979 & 0.942 & 0.979 & 0.967 & 0.969 & 0.978 & 0.965 & \\
diabetes & \textbf{0.806} & 0.732 & 0.737 & 0.691 & 0.737 & 0.734 & 0.733 & 0.734 & 0.697 & \\
fourclass & \textbf{0.988} & 0.665 & \textbf{0.988} & \textbf{0.988} & \textbf{0.988} & \textbf{0.988} & \textbf{0.988} & \textbf{0.988} & \textbf{0.988} & \\
german.numer & \textbf{0.819} & 0.640 & 0.687 & 0.686 & 0.672 & 0.665 & 0.657 & 0.686 & 0.692 & \\
heart & \textbf{0.918} & 0.822 & 0.834 & 0.751 & 0.839 & 0.787 & 0.783 & 0.833 & 0.799 & \\
ionosphere & \textbf{0.990} & 0.810 & 0.798 & 0.763 & 0.849 & 0.804 & 0.814 & 0.798 & 0.863 & \\
liver-disorders & \textbf{0.763} & 0.682 & 0.659 & 0.707 & 0.698 & 0.691 & 0.676 & 0.688 & 0.715 & \\
sonar & \textbf{0.996} & 0.714 & 0.717 & 0.892 & 0.729 & 0.702 & 0.709 & 0.717 & 0.735 & \\
splice & \textbf{0.927} & 0.738 & 0.724 & 0.829 & 0.716 & 0.717 & 0.718 & 0.723 & 0.742 & \\
\bottomrule
\end{tabular}
\end{center}
\caption{ Comparison of 2-dimensional projections discriminativeness.}
\label{tab:vis}
\end{sidewaystable}

During these experiments MELM achieved essentially better scores than any other tested method (see Table~\ref{tab:vis}). 
Solutions were about 10\% better under our metric and this difference is consistent over all considered datasets. 
In other words MELM finds two-dimensional representations of our data using just linear projection where classes overlap to a significantly smaller degree than using PCA, cPCA, 2ePCA, pPCA, ICA, NMF, FA or GEM. It is also worth noting that Factor Analysis, as the only method which does not require orthogonality of resulting projection vectors did a really bad job while working with fourclass data even though these samples are just two-dimensional.

\section{Conclusions}

In this paper we construct Maximum Entropy Linear Manifold (MELM), a method of learning discriminative low-dimensional representation which can be used for both classification purposes
as well as a visualization preserving classes separation. Proposed model has important theoretical properties including affine transformations invariance, connections
with PCA as well as bounding the expected balanced misclassification error. During evaluation we show that for moderate size problems MELM can be efficiently
optimized using simple first-order optimization techniques. Obtained results confirm that such an approach leads to highly discriminative transformation, better
than obtained by 8 compared solutions.

\subsubsection*{Acknowledgments.} The work has been partially financed by National Science Centre Poland grant no. 2014/13/B/ST6/01792.

\end{document}